\documentclass[12pt]{article}
\usepackage{amsmath}
\usepackage{times}
\usepackage{graphicx}
\usepackage{color}
\usepackage{multirow}
\usepackage[authoryear]{natbib}
\usepackage{rotating}
\usepackage{bbm}
\usepackage{latexsym}
\usepackage{amssymb}
\usepackage{mathtools}
\usepackage{booktabs}

\usepackage{amsmath}
\usepackage{amssymb}
\usepackage{mathtools}
\usepackage{amsthm}
\usepackage{float}

\usepackage{algorithm}
\usepackage{algorithmic}

\usepackage{appendix}

\usepackage{thm-restate}
\usepackage{amsmath}

\usepackage{graphicx}
\usepackage{subfigure}

\newtheorem{theorem}{Theorem}[section]
\newtheorem{proposition}[theorem]{Proposition}
\newtheorem{lemma}[theorem]{Lemma}

\newcommand{\argmax}{\mathop{\rm arg~max}\limits}
\newcommand{\argmin}{\mathop{\rm arg~min}\limits}

\newcommand{\captionfonts}{\normalsize}

\makeatletter  
\long\def\@makecaption#1#2{%
  \vskip\abovecaptionskip
  \sbox\@tempboxa{{\captionfonts #1: #2}}%
  \ifdim \wd\@tempboxa >\hsize
    {\captionfonts #1: #2\par}
  \else
    \hbox to\hsize{\hfil\box\@tempboxa\hfil}%
  \fi
  \vskip\belowcaptionskip}
\makeatother   

\begin{document}
\hspace{13.9cm}1

\ \vspace{20mm}\\

{\LARGE A Fast Algorithm for the Real-Valued Combinatorial Pure Exploration of Multi-Armed Bandit}

\ \\
{\bf \large Shintaro Nakamura $^{\displaystyle 1, \displaystyle 2}$ and Masashi Sugiyama $^{\displaystyle 2, \displaystyle 1}$}\\
{$^{\displaystyle 1}$ The University of Tokyo}\\
{$^{\displaystyle 2}$ RIKEN AIP}\\
%


\thispagestyle{empty}
\markboth{}{NC instructions}
\ \vspace{-0mm}\\
%
\begin{center} {\bf Abstract} \end{center}
We study the real-valued combinatorial pure exploration problem in the stochastic multi-armed bandit (R-CPE-MAB). We study the case where the size of the action set is polynomial with respect to the number of arms. In such a case, the R-CPE-MAB can be seen as a special case of the so-called transductive linear bandits. 
We introduce an algorithm named the combinatorial gap-based exploration (CombGapE) algorithm, whose sample complexity upper bound matches the lower bound up to a problem-dependent constant factor. We numerically show that the CombGapE algorithm outperforms existing methods significantly in both synthetic and real-world datasets.

\section{Introduction}\label{IntroductionSection}
The stochastic multi-armed bandit (MAB) model is one of the most popular models for action-making problems where we investigate trade-offs between exploration and exploitation in a stochastic environment. In this model, we are given a set of stochastic arms associated with unknown distributions. Whenever an arm is pulled, it generates a reward sampled from the corresponding distribution. The MAB model is mainly used for two objectives. One is regret minimization, or cumulative reward maximization \citep{Auer2002,Bubeck2012,Auer2002B}, where the player tries to maximize its average rewards by sequentially pulling arms by balancing \emph{exploring} all the distributions and \emph{exploiting} the most rewarding ones. The other is pure exploration \citep{Audibert2010,Gabillon2012}, where the player tries to identify the optimal arm with high probability with as few samples as possible. \par
On the other hand, it is well known that many real-world problems can be modeled as linear optimization problems, such as the shortest path problem, the optimal transport problem, the minimum spanning tree problems, the traveling salesman problem, and many more \citep{Sniedovich2006,villani2008,Pettie2002,Gutin2001}. Abstractly, these linear optimization problems can be formulated as follows:
\begin{equation*}
\begin{array}{ll@{}ll}
\text{maximize}_{\boldsymbol{\pi}}  & \boldsymbol{\mu}^{\top}\boldsymbol{\pi}&\\
\text{subject to}& \boldsymbol{\pi}\in \mathcal{A}\subset\mathbb{R}^d,   &
\end{array} 
\end{equation*} 
where $\boldsymbol{\mu}\in\mathbb{R}^d$ is a given vector that specifies the cost, $d (\geq 2)$ is a positive integer, $\top$ denotes the transpose, and $\mathcal{A}$ is a set of feasible solutions. For instance, if we see the shortest path problem in Figure~\ref{ShortestPathFigure}, each edge $i \in \{1, \ldots, 7\}$ has a cost $\mu_i$ and $\mathcal{A} = \{(1, 0, 1, 0, 0, 1, 0), (0, 1, 0, 1, 0, 1, 0), (0, 1, 0, 0, 1, 0, 1)\}$. The optimal transport problem shown in Figure \ref{OT_figure} can also be formulated similarly to the above. We have five suppliers and four demanders. Each supplier $i$ has $s_i$ goods to supply. Each demander $j$ wants $d_j$ goods. Each edge $\mu_{ij}$ is the cost to transport goods from supplier $i$ to demander $j$. Our goal is to minimize $\sum_{i = 1}^5 \sum_{j = 1}^4 \pi_{ij}\mu_{ij}$ where $\pi_{ij}(\geq 0)$ is the number of goods transported to demander $j$ from supplier $i$. 

\begin{figure}[t]
    \centering
    \includegraphics[width = \linewidth/3]{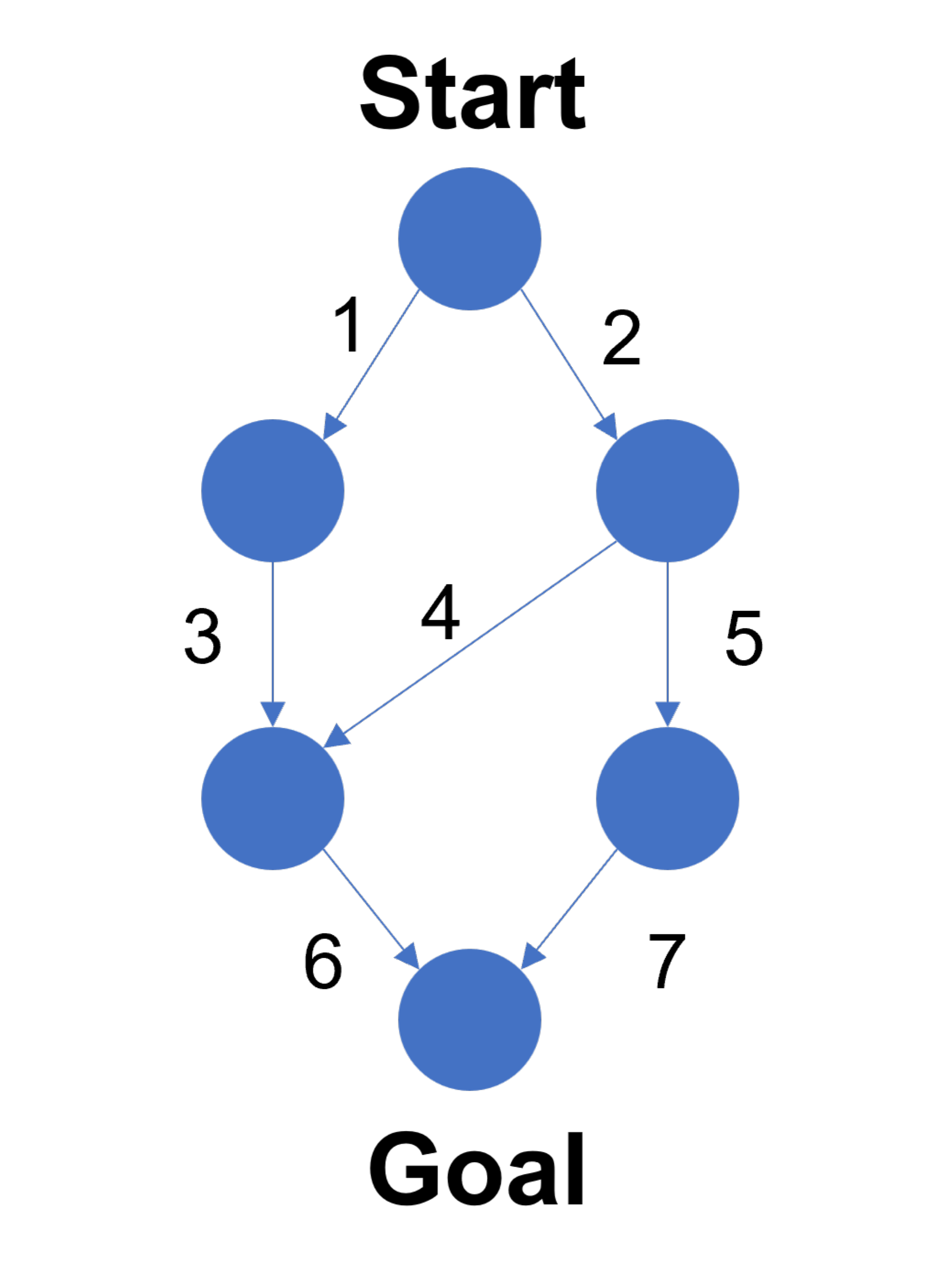}
    \caption{A simple sketch of the shortest path problem. One candidate of $\boldsymbol{\pi}$ can be $\boldsymbol{\pi} = \left(1, 0, 1, 0, 0, 1, 0 \right)^{\top}$.}
    \label{ShortestPathFigure}
\end{figure}
\begin{figure}
    \centering
    \includegraphics[width = \linewidth]{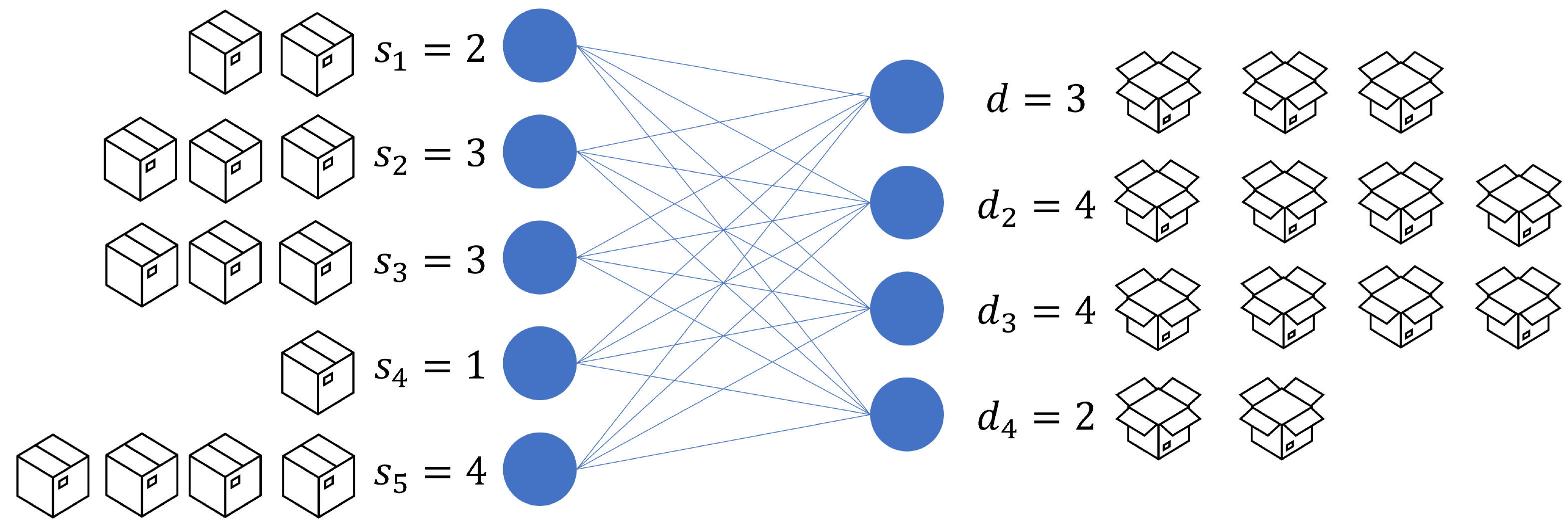}
    \caption{A simple sketch of the optimal transport problem. One candidate of $\boldsymbol{\pi}$ can be 
        $\boldsymbol{\pi} = \begin{psmallmatrix}
        2 & 0 & 0 & 0 \\
        1 & 2 & 0 & 0 \\
        0 & 2 & 1 & 0 \\
        0 & 0 & 1 & 0 \\
        0 & 0 & 2 & 2 \\
        \end{psmallmatrix}.
        $ }
\label{OT_figure}
\end{figure}

In real-world applications, the cost can often be random variables. For instance, in the shortest path problem, each edge (road) can be congested, and therefore the cost stochastically changes. Here, our main interest is to identify the best action $\boldsymbol{\pi}^{*}\in\mathcal{A}$ which maximizes the expectation of the objective $\boldsymbol{\mu}^{\top}\boldsymbol{\pi}$ in the MAB framework. \par
There are many MAB frameworks where the goal is to identify the best action for a linear objective in a stochastic environment \citep{Soare2014,Xu2018,Kuroki2020,Rejwan2020,Jourdan2021,SChen2014,YihanDu2021} (Table \ref{Taxonomy_of_framework}). Our work focuses on the combinatorial pure exploration of the MAB (CPE-MAB) framework \citep{SChen2014,WangAndZhu,Gabillon16}. In CPE-MAB, the player pulls a certain arm $i$ in each round and observes its reward. The player's goal is to identify the best action $\boldsymbol{\pi}^{*} \in\mathcal{A}$ with the smallest number of rounds. \par 
Most of the existing works in CPE-MAB \citep{SChen2014,WangAndZhu,Gabillon16,LChen2017,YihanDu2021,LChen2016} assume $\mathcal{A}\subset\{0, 1\}^d$, which means the player's objective is to identify the best action which maximizes the sum of the expected rewards. Although we can apply this model to the shortest path problem \citep{Sniedovich2006}, top-$K$ arms identification \citep{Kalyanakrishnan2010}, matching \citep{Gibbons1985}, and spanning trees \citep{Pettie2002}, we cannot apply it to problems where $\mathcal{A} \subset \mathbb{R}^d$ such as the optimal transport problem \citep{villani2008}, the knapsack problem \citep{dantzig2007}, and the production planning problem \citep{Pochet2010}. \par
In this paper, we study the real-valued CPE-MAB (R-CPE-MAB), where the action set is $\mathcal{A}\subset \mathbb{R}^{d}$. For the R-CPE-MAB, \citet{nakamura2023thompson} proposed a Thompson sampling-based algorithm named the Generalized Thompson Sampling Explore (GenTS-Explore) algorithm, which works even when the size of the action set is exponentially large in $d$. At each round $t$, the GenTS-Explore algorithm samples $M(\delta, t)$ estimations of $\boldsymbol{\mu}$, $\left\{ \tilde{\theta}_{i}(t) \right\}_{i = 1, \ldots, M(\delta, t)}$, from a Gaussian distribution, and computes $\left\{ \tilde{\boldsymbol{\pi}}^{i}(t) \right\}_{i = 1, \ldots, M(\delta, t)}$ by using the \emph{offline oracle} that computes $\argmax_{\boldsymbol{\pi} \in \mathcal{A}} \boldsymbol{\nu}^{\top} \boldsymbol{\pi}$ efficiently, once $\boldsymbol{\nu}$ is given. The GenTS-Explore algorithm will terminate if $\hat{\boldsymbol{\pi}} = \tilde{\boldsymbol{\pi}}^{k}(t)$ for all $k \in \{1, \ldots, M(\delta, t)  \}$, and output $\hat{\boldsymbol{\pi}}(t)$ as the final answer. However, since $M(\delta, t)$ is set large, the sample complexity becomes large practically.
On the other hand, if we assume the size of the action set is polynomial in $d$, the R-CPE-MAB reduces to a special case of the transductive linear bandits \citep{Fiez2019,KatzSamuels2020}. To the best of our knowledge, we mainly have two algorithms from the transductive linear bandits that work for the R-CPE-MAB, the Randomized Adaptive Gap Elimination (RAGE) \citep{Fiez2019} and Peace \citep{KatzSamuels2020} algorithms. One advantage of these algorithms is that the upper bounds match the lower bound up to a logarithmic factor. They both have multiple phases and in each phase $P$, they define $T(P) \in \mathbb{N}$ and pull each arm $i$ $T_{i}(P)$ times, where $\sum\limits_{i = 1}^{d} T_{i}(P) = T(P)$. However, $T(P)$ is often larger than necessary in practice, and eventually, has a large sample complexity. \par
In this paper, we propose an algorithm named the CombGapE (Combinatorial Gap-Based Exploration) algorithm for the R-CPE-MAB for the case where the size of the action set $\mathcal{A}$ is polynomial in $d$. In Appendix \ref{PolynomialAppendix}, we show some examples where the size of $\mathcal{A}$ is polynomial in $d$ due to some prior knowledge or the nature of the problem instance. Although the upper bound of the sample complexity of the CombGapE algorithm has an extra problem-dependent constant factor compared with a lower bound, it significantly outperforms existing methods for the R-CPE-MAB in practice.

The remainder of this paper is structured as follows.
In Section \ref{ProblemFormulation}, we formally introduce the R-CPE-MAB. In Sections \ref{ArmSelectionStrategySection}, we explain why naively modifying existing methods for CPE-MAB may not be suited for R-CPE-MAB, and give a new arm selection strategy for R-CPE-MAB. In Section \ref{Algorithm_and_TheoreticalAnalysis}, we show our algorithm inspired by \citet{Xu2018} and theoretically analyze the sample complexity of our algorithm. Finally, we experimentally compare our algorithm to some existing works in Section~\ref{ExperimentSection} and conclude our work in Section \ref{Conclusion}. \par

\textbf{Related Works and the Location of this Work} \par

\begin{table*}[t] \label{Taxonomy_of_framework}
\caption{Taxonomy of MAB problems. We compare the pure exploration in linear bandits (PE-LB), combinatorial pure exploration with full-bandit linear feedback, combinatorial pure exploration with semi-bandit linear feedback (CPE-S-LB), combinatorial pure exploration of multi-armed bandit (CPE-MAB), and the real-valued combinatorial pure exploration of multi-armed bandit (R-CPE-MAB).}
\label{model-table}
\begin{center}
\begin{tabular}{ccccc}
\toprule
& Player's behavior & Observation & Characteristic of $\mathcal{A}$ \\
\midrule
PE-LB  & Play an action $\boldsymbol{\pi}(t)$ & $\boldsymbol{\mu}^{\top}\boldsymbol{\pi}(t) + \epsilon_t$  & $\mathcal{A}\subset\mathbb{R}^d$\\
CPE-F-LB & Play an action $\boldsymbol{\pi}(t)$ & $\sum_{i:\boldsymbol{\pi}_i(t) = 1}r_i$ & $\mathcal{A}\subset\{0, 1\}^d$ \\
CPE-S-LB & Play an action $\boldsymbol{\pi}(t)$ & 
\begin{tabular}{c}
    $\{r_i \mid \forall i \in [d],$ \\
    $\pi_i(t) = 1 \}$
\end{tabular}   & $\mathcal{A}\subset\{0, 1\}^d$\\
CPE-MAB & Pull an arm $i$ & $r_i$  & $\mathcal{A}\subset\{0, 1\}^d$ \\
\begin{tabular}{c}
    R-CPE-MAB \\
    (this work)
\end{tabular}  & Pull an arm $i$ & $r_i$ & $\mathcal{A}\subset\mathbb{R}^d$ \\
\bottomrule
\end{tabular}
\end{center}
\end{table*}

Here, we introduce models that investigate the best action identification in linear objective optimization with an MAB framework (Table \ref{model-table}). Then, we briefly explain the location of this work in the literature. \par
\textbf{Pure Exploration in Linear Bandits (PE-LB)} In PE-LB \citep{Soare2014,Xu2018}, we are given in advance a set of actions whose size is polynomial in $d$. Then, in each time step~$t$, the player chooses an action $\boldsymbol{\pi}(t)$ from the action set and observes a reward with noise, i.e., $r(t) = \boldsymbol{\mu}^{\top}\boldsymbol{\pi}(t) + \epsilon(t)$, where $\epsilon(t)$ is a random noise from a certain distribution. Since $\boldsymbol{\mu}$ is typically treated as an unknown vector, many algorithms estimate $\boldsymbol{\mu}$ by leveraging the sequence of action selections. For instance, \citet{Soare2014} used the least-squares estimator to estimate $\boldsymbol{\mu}$. \par
\textbf{Combinatorial Pure Exploration with Full-Bandit Linear Feedback (CPE-F-BL)} In CPE-F-BL \citep{Kuroki2020,Rejwan2020}, arms are often called \emph{base arms}, and the actions are often called \emph{super arms}. Super arms are subsets of base arms, therefore $\mathcal{A}\subset\{0, 1\}^d$. In each time step, the player chooses an action and observes the sum of rewards from the base arms involved in the chosen action. Therefore, CPE-F-BL can be seen as an instance of PE-LB. However, since the running time of existing methods for PE-LB has polynomial dependence on the size of action space, they are not suitable for CPE-F-BL, where the size of the action space $\mathcal{A}$ can be exponentially large with respect to $d$. To cope with this problem, algorithms specifically designed for CPE-F-BL have been explored in the literature \citep{Kuroki2020,Rejwan2020,Du2021}.\par
\textbf{Combinatorial Pure Exploration with Semi-Bandit Linear Feedback (CPE-S-BL)} CPE-S-BL \citep{Jourdan2021} is a framework similar to CPE-F-BL: In each round, the player chooses an action from the set of actions (super arms) $\mathcal{A}\subset\{0, 1\}^d$, which is a subset of base arms. Yet, the observation is different. The player observes rewards $r_i$ for every base arm $i$ where $\pi_i = 1$. For instance, in Figure \ref{ShortestPathFigure}, if the player chooses $\boldsymbol{\pi} = (1, 0, 1, 0, 0, 1, 0)$, she can observe the rewards of base arms 1, 3, and 6. \par 

\textbf{Transductive Linear Bandit Problem} \citet{Fiez2019} introduced a framework named the \emph{transductive linear bandit problem}: given a set of \emph{measurement vectors} $\mathcal{X} \subset \mathbb{R}^{d}$, a set of actions $\mathcal{A} \subset \mathbb{R}^{d}$, a fixed confidence $\delta$, and an unknown vector $\boldsymbol{\mu} \subset \mathbb{R}^{d}$, a player tries to identify $\argmax_{\boldsymbol{\pi}\in \mathcal{A}} \boldsymbol{\mu}^{\top} \boldsymbol{\pi}$ with probability more than $1 - \delta$ by sequentially choosing a measurement vector $\boldsymbol{x}(t) \in \mathcal{X}$ in each round $t$ and observing a stochastic random variable whose expected value is ${\boldsymbol{x}(t)}^{\top} \boldsymbol{\mu}$. When $\mathcal{X} = \mathcal{A}$, this setting reduces to the linear bandits, and when $\mathcal{X}$  is the standard basis vectors and $\mathcal{A} \subset \mathbb{R}^{d}$, this setting reduces to the R-CPE-MAB. Therefore, we can apply existing algorithms from the transductive linear bandit, such as the RAGE (Randomized Adaptive Gap Elimination) \citep{Fiez2019} and Peace \citep{KatzSamuels2020} algorithms for the R-CPE-MAB. \par
\textbf{Contribution of this Work}\par
In this work, we propose the CombGapE (combinatorial gap-based exploration) algorithm that significantly outperforms existing methods for the R-CPE-MAB in practice. Also, we theoretically analyze the sample complexity upper bound of it and show that it matches a lower bound up to a problem-dependent constant factor. We also discuss the difference between the R-CPE-MAB and the ordinary CPE-MAB and show both quantitatively and qualitatively that naively applying algorithms in the literature of the ordinary CPE-MAB is not a good choice for the R-CPE-MAB. Finally, we numerically show that the CombGapE algorithm outperforms existing methods, taking an example of the knapsack problem.

\section{Problem Formulation} \label{ProblemFormulation}
In this section, we formally define our R-CPE-MAB model similar to \citet{SChen2014}. Suppose we have $d$ arms, numbered $1, \ldots, d$. Assume that each arm $s\in[d]$ is associated with a reward distribution $\phi_s$, where $[d] = \{1, \ldots, d\}$. We assume all reward distributions have $R$-sub-Gaussian tails for some known constant $R>0$. 
Formally, a random variable $X$ is said to be $R$-sub Gaussian if $\mathbb{E}[\exp (tX - t\mathbb{E}[X])] \leq \exp (R^2t^2/2)$ is satisfied for all $t\in \mathbb{R}$. 
It is known that the family of $R$-sub-Gaussian tail distributions includes all distributions that are supported on $[0, R]$ and also many unbounded distributions such as Gaussian distributions with variance $R^2$ \citep{Rivasplata2012}.
Let $\boldsymbol{\mu} = (\mu_1, \ldots, \mu_d)^{\top}$ denote the vector of expected rewards, where each element $\mu_s = \mathbb{E}_{X\sim\phi_s}[X]$ denotes the expected reward of arm $s$. We denote the number of times arm $s$ is pulled before round $t$ by $T_s(t)$, and by $\boldsymbol{\hat{\mu}}(t) = (\hat{\mu}_1(t), \ldots, \hat{\mu}_d(t))^{\top}$ the vector of sample means of each arm in round $t$.\par
We define the action class $\mathcal{A} = \{\boldsymbol{\pi}^1, \ldots, \boldsymbol{\pi}^K\ | \boldsymbol{\pi}^1, \ldots, \boldsymbol{\pi}^{K} \in \mathbb{R}^d \} $ as the set of all actions whose size is $K$. We assume $K$ is polynomial in $d$. Let $\boldsymbol{\pi}^* = \argmax_{\boldsymbol{\pi}\in\mathcal{A}}\boldsymbol{\mu}^{\top}\boldsymbol{\pi}$ denote the optimal member in the action class $\mathcal{A}$ which maximizes $\boldsymbol{\mu}^{\top}\boldsymbol{\pi}$. Let $a^{*}$ be the index of $\boldsymbol{\pi}^{*}$, i.e., $\boldsymbol{\pi}^{*} = \boldsymbol{\pi}^{a^{*}}$. We denote the true gap and the estimated gap in time step $t$ by $\Delta_{ij} = \boldsymbol{\mu}^{\top}(\boldsymbol{\pi}^{i} - \boldsymbol{\pi}^j)$ and $\hat{\Delta}_{ij}(t) = {\hat{\boldsymbol{\mu}}(t)}^{\top}(\boldsymbol{\pi}^{i} - \boldsymbol{\pi}^{j})$, respectively. Moreover, for all $s \in [d]$, we define $\Delta_{(s)} = \argmax_{\boldsymbol{\pi} \in \mathcal{A} \setminus \{ \boldsymbol{\pi}^{*} \}} \frac{ \boldsymbol{\mu}^{\top} \left( \boldsymbol{\pi}^{*} - \boldsymbol{\pi}  \right)}{\left| \pi^{*}_{s} - \pi_{s} \right|}$. \par
The player's objective is to identify $\boldsymbol{\pi}^{*}$ from $\mathcal{A}$ by playing the following game. At the beginning of the game, the action class $\mathcal{A}$ is revealed while the reward distributions $\{\phi_{s}\}_{s\in[d]}$ are unknown to her. Then, the player pulls an arm over a sequence of rounds; in each round $t$, she pulls an arm $p(t)\in[d]$ and observes a reward $r_{p(t)}$ sampled from the associated reward distribution $\phi_{p(t)}$. The game continues until a certain stopping condition is satisfied. After the game finishes, the player outputs an action $\boldsymbol{{\pi}}_{\mathrm{out}}\in\mathcal{A}$. Let ${a}_{\mathrm{out}}$ be the index of $\boldsymbol{{\pi}}_{\mathrm{out}}$, i.e., $\boldsymbol{{\pi}}_{\mathrm{out}} = \boldsymbol{\pi}^{{a}_\mathrm{out}}$.\par
The problem is to design an algorithm to find an action $\boldsymbol{\pi}_{\mathrm{out}}$ which satisfies
\begin{equation}
    \Pr\left[ \boldsymbol{\pi}^* = \boldsymbol{\pi}_{\mathrm{out}} \right] \geq 1 - \delta, \nonumber 
\end{equation}
as fast as possible. We denote by $\tau$ the round an algorithm terminated. We use the term sample complexity as the round an algorithm terminated.

\section{The Arm Selection Strategy} \label{ArmSelectionStrategySection}
In this section, we discuss the arm selection strategy qualitatively and quantitatively.
We first qualitatively show that applying or naively modifying existing works in CPE-MAB for R-CPE-MAB may not be a good choice. Then, we discuss the arm selection strategy quantitatively by looking at the confidence bound of the estimated gap between actions and propose a new arm selection strategy for R-CPE-MAB.
\subsection{Limitation of Existing Works in CPE-MAB} \label{limitation_of_Existing_Work}
Here, we first briefly explain what some of the existing algorithms in CPE-MAB \citep{SChen2014,Gabillon16,WangAndZhu} are doing at a higher level. First, in each round $t$, they output the action $\hat{\boldsymbol{\pi}}(t)$ which maximizes ${\hat{\boldsymbol{\mu}}(t)}^{\top}\boldsymbol{\pi}$. As we have seen in the example in Figure \ref{ShortestPathFigure}, we denote by $\hat{S}$ the set that contains all the arms ``used'' in $\hat{\boldsymbol{\pi}}(t)$. Next, they output another action $\Tilde{\boldsymbol{\pi}}(t)$, which is potentially the best action by considering a confidence bound on arms. Similarly, we denote by $\Tilde{S}$ the set of arms that contains all the arms ``used'' in $\Tilde{\boldsymbol{\pi}}(t)$. 
Finally, they choose the arm $s$ that is in $\hat{S}\oplus\Tilde{S} = (\hat{S}\setminus\Tilde{S})\cup(\Tilde{S}\setminus\hat{S})$ with the least number of pulls. This means they are choosing an arm $s$ where $\hat{{\pi}}_s(t) \neq \tilde{\pi}_s(t)$. They repeat this until some stopping condition for the algorithm is satisfied.  \par
In R-CPE-MAB, we can no longer think $\hat{S}\oplus\Tilde{S} =(\hat{S}\setminus\Tilde{S})\cup(\Tilde{S}\setminus\hat{S})$ since we are thinking of $\mathcal{A}\subset\mathbb{R}^d$, and therefore have to think of other arm selection strategies. One naive way to modify their methods to R-CPE-MAB is to choose the arm $s$ with the least number of times it was pulled among the set of arms $\{ s \in [d] \ | \ \hat{\pi}_s(t) \neq \tilde{\pi}_s(t) \}$. However, this may not be an efficient strategy to pull arms in R-CPE-MAB. To explain this, suppose we have two actions $\boldsymbol{\pi}^{1} = (100, 0, 0.1)$ and $\boldsymbol{\pi}^{2} = (0, 100, 0.2)$, and the stopping condition for an algorithm is not satisfied. Also, let us assume that we are pretty sure that $\hat{{\mu}}_1(t) \in (0.9 - 0.01, 0.9+0.01)$, $\hat{{\mu}}_2(t) \in (0.9 - 0.01, 0.9+0.01)$, and $\hat{{\mu}}_3(t) \in (0.9 - 0.01, 0.9+0.01)$ using some concentration inequality. Here, although the confidence interval is all the same when we estimate the gap between actions 1 and 2, i.e., $\boldsymbol{\mu}^{\top}(\boldsymbol{\pi}^1 - \boldsymbol{\pi}^2)$, the uncertainty of arms 1 and 2 will be amplified by 100 times where the uncertainty of arm 3 will be amplified by only 0.1 times. This example suggests that we must consider how \emph{important} a certain arm is to estimate the gaps between actions in R-CPE-MAB, and not simply pull the arm with the least number of times it was pulled among the set of arms $\{ s \in [d] \ | \ \hat{\pi}_s(t) \neq \tilde{\pi}_s(t) \}$. In Appendix \ref{EWLimitationAppendix}, we discuss the limitation of other works on CPE-MAB \citep{YihanDu2021,LChen2016,LChen2017}, which are close to the above discussion. 

\subsection{Confidence Bounds and the Arm Selection Strategy}
In Section \ref{limitation_of_Existing_Work}, we have seen that the uncertainty of arms will be amplified when we estimate gaps between actions. We discuss this quantitatively here.\par
First, we bound the gap of two actions $\boldsymbol{\pi}^{k}$ and $\boldsymbol{\pi}^{l}$ with a concentration inequality. Define $\beta_{kl}(t)$ as
\begin{align}
    \beta_{kl}(t) = R \sqrt{\frac{1}{2} \sum_{s = 1}^{d}\frac{(\pi^k_s - \pi^l_s)^2}{T_s(t)}\log\frac{2K^2 t^2}{\delta}} \label{beta}.
\end{align}
Then, we have the following proposition.
\begin{proposition} \label{KeyProposition}
    Let $T_s(t)$ be the number of times arm $s$ has been pulled before round $t$. Then, for any $t\in\mathbb{N}$ and $\boldsymbol{\pi}^k,\boldsymbol{\pi}^l \in \mathcal{A}$, with probability at least $1 - \delta$, we have
    \begin{align}
            \Pr \left( \left|\Delta_{kl} - \hat{\Delta}_{kl}(t) \right| \leq \beta_{kl}(t) \right) \geq 1 - \delta.
    \end{align}
\end{proposition}
We show the proof in Appendix \ref{ProofKeyProposition}.
If we set $\delta$ small, Lemma \ref{KeyProposition} shows that the estimated gap between two actions is close to the true gap with high probability. We regard
\begin{align}
    \hat{\Delta}_{kl}(t) + \beta_{kl}(t)
\end{align}
as an upper confidence bound of the estimated gap between actions $\boldsymbol{\pi}^k$ and $\boldsymbol{\pi}^l$. \par
Assume we want to estimate the gap between two actions $\boldsymbol{\pi}^{k}$ and $\boldsymbol{\pi}^{l}$. Since we want to estimate it as fast as possible, we want to pull an arm that makes the confidence bound $\beta_{kl}(t)$ smallest.
\begin{align}
    p(t) = \argmin_{u\in[d]}  \sum_{s = 1}^{d}\frac{(\pi^k_s - \pi^l_s)^2}{T_s(t) + \boldsymbol{1}[s = u]}, \label{p(t)_def}
\end{align}
where $\boldsymbol{1}[\cdot]$ denotes the indicator function.
Then, the following proposition holds.
\begin{restatable}[]{proposition}{ArmSelectionStrategyProposition} \label{ArmSelectionStrategyProposition}
    $p(t)$ in (\ref{p(t)_def}) can be written as follows:
    \begin{align}
        p(t) = \argmax_{s\in[d]} \frac{(\pi^{k}_s - \pi^{l}_{s})^2}{T_s(t)(T_s(t) + 1)}. \label{p(t)_def_2}
    \end{align}
\end{restatable}
We show the proof in Appendix \ref{ArmSelectionStrategyPropositionProof}. Intuitively speaking, we can say computing (\ref{p(t)_def_2}) is considering the \emph{importance} of each arm to estimate the gap between $\boldsymbol{\pi}^{k}$ and $\boldsymbol{\pi}^{l}$. If $({\pi}_s^{k} - {\pi}_s^{l})^2$ is large, that means the uncertainty of arm $s$ can be amplified largely, and arm $s$ needs to be pulled many times to reduce the gap between actions $\boldsymbol{\pi}^{k}$ and $\boldsymbol{\pi}^{l}$. Additionally, the arm selection strategy (\ref{p(t)_def_2}) is equivalent to pulling the arm with the least number of times pulled among $\{ i \in [d] \ | \ {\pi}^{k}_{s} \neq {\pi}^{l}_{s} \}$ in CPE-MAB. Therefore, we can see the arm selection strategy (\ref{p(t)_def_2}) as a generalization of the arm selection strategies in \citet{SChen2014}, \citet{Gabillon16}, and \citet{WangAndZhu}.

\section{CombGapE Algorithm and Theoretical Analysis}\label{Algorithm_and_TheoreticalAnalysis}
In this section, we introduce our algorithm (Algorithm \ref{OurAlgorithm}) inspired by \citet{Xu2018}. We name it CombGapE (Combinatorial Gap-based Exploration algorithm). We first show that, given a confidence parameter $\delta$, the CombGapE algorithm identifies the best action with a probability of at least $1 - \delta$. Then, we establish the sample complexity of the CombGapE algorithm.

\subsection{CombGapE Algorithm}
We show our algorithm CombGapE in Algorithm \ref{OurAlgorithm}.
In each round, it chooses two actions $\boldsymbol{\pi}^{i(t)}$ and $\boldsymbol{\pi}^{j(t)}$. We can say $\boldsymbol{\pi}^{i(t)}$ is the action that is most likely to be the best since it has the largest estimated reward. Also, we can say that $\boldsymbol{\pi}^{j(t)}$ is an action that is potentially the best since the upper confidence bound of the estimated gap between $\boldsymbol{\pi}^{i(t)}$ and $\boldsymbol{\pi}^{j(t)}$ is positive, which implies that there is still a chance that $\boldsymbol{\pi}^{j(t)}$ is better that $\boldsymbol{\pi}^{i(t)}$.
Then, as we discussed in Section \ref{ArmSelectionStrategySection}, CombGapE pulls the arm $p(t)$ that most reduces the confidence bound $\beta_{i(t) j(t)}(t)$ (line \ref{pull_p(t)_line}). \par

\begin{algorithm}[]
   \caption{CombGapE Algorithm}
   \label{OurAlgorithm}
\begin{algorithmic}[1]
   \STATE {\bfseries Input:} Confidence level $\delta$ and an action set~$\mathcal{A}$
   \STATE {\bfseries Output:} action $\boldsymbol{\pi}_{\mathrm{out}} \in \mathcal{A}$
   \STATE $t\leftarrow 1$
   \STATE \texttt{// Initialization phase}
   \FOR{$s=1$ {\bfseries to} $d$}
    \STATE Observe $r_s$
    \STATE $t \leftarrow t + 1$
   \ENDFOR
   \FOR{$t = d + 1$, \ldots}
   \STATE \texttt{// Select which gap to examine} 
   \STATE $(i(t), j(t), B(t)) \leftarrow \textrm{Select-Ambiguous-Action($t$)}$ (see Algorithm \ref{SelectAmbiguousAction})
   \STATE \texttt{// Check the stopping condition} 
    \IF{$B(t)\leq 0$}
   \STATE \textbf{return} $\boldsymbol{\pi}^{i(t)}$ as the best action $\boldsymbol{\pi}_{\mathrm{out}}$
    \ENDIF
   \STATE \texttt{// Pull~an~arm~based~on (\ref{p(t)_def_2})} 
    \STATE Pull $p(t) = \argmax_{s\in[d]]} \frac{(\pi^{i(t)}_s - \pi^{j(t)}_{s})^2}{T_s(t)(T_s(t) + 1)}$\label{pull_p(t)_line}
    \STATE Observe $r_{p(t)}$, and update the number of pulls:\\ $T_{p(t)}(t + 1)\leftarrow T_{p(t)} (t) + 1$ and $T_{e}(t + 1) \leftarrow T_{e}(t)$ for all $e \neq p(t)$ 
   \ENDFOR
\end{algorithmic}
\end{algorithm}

\newcommand\algorithmicprocedure{\textbf{procedure}}
\newcommand{\algorithmicendprocedure}{\algorithmicend\ \algorithmicprocedure}

\makeatletter
\newcommand\PROCEDURE[3][default]{%
  \ALC@it
  \algorithmicprocedure\ \textsc{#2}(#3)%
  \ALC@com{#1}%
  \begin{ALC@prc}%
}
\newcommand\ENDPROCEDURE{%
  \end{ALC@prc}%
  \ifthenelse{\boolean{ALC@noend}}{}{%
    \ALC@it\algorithmicendprocedure
  }%
}
\newenvironment{ALC@prc}{\begin{ALC@g}}{\end{ALC@g}}
\makeatother

\begin{algorithm}[]
   \caption{\texttt{Select-Ambiguous-Action}(t)}
   \label{SelectAmbiguousAction}
    \begin{algorithmic}[1] 
        \PROCEDURE{Select-Ambiguous-Action}{}
            \STATE Compute $\hat{\mu}_{s}(t)$ for every $s\in[d]$
            \STATE $i(t) \leftarrow \argmax_{i \in [K]} {\boldsymbol{\pi}^i}^{\top}\boldsymbol{\hat{\mu}}(t)$ 
            \STATE $j(t) \leftarrow \argmax_{j \in [K]} \hat{\Delta}_{j i(t)}(t) + \beta_{j i(t)}(t)$
            \STATE $B(t) \leftarrow \max_{j \in [K]} \hat{\Delta}_{j i(t)}(t) + \beta_{j i(t)}(t)$ \label{definition_of_B}
            \STATE \textbf{return} $(i(t), j(t), B(t))$
        \ENDPROCEDURE
    \end{algorithmic}
\end{algorithm}

\subsection{Accuracy and the Sample Complexity} \label{Accuracy_and_SampleComplexity_Section}
Here, in Theorem \ref{identification_accuracy}, we first show that given a confidence parameter $\delta$, our algorithm misidentifies the best action with probability at most $\delta$.
\begin{theorem} \label{identification_accuracy}
    The output of Algorithm \ref{OurAlgorithm} $\boldsymbol{\pi}^{a_{\mathrm{out}}}$ satisfies the following condition:
    \begin{align}
        \Pr\left[\boldsymbol{\pi}^{*} = \boldsymbol{\pi}_{\mathrm{out}}\right] \geq 1 - \delta.
    \end{align}
\end{theorem}
We show the proof in Appendix \ref{ProofIdentificationAccuracy}.\par
Next, we show an upper bound of the sample complexity of the CombGapE algorithm.
\begin{theorem} \label{sample_complexity}
    Assume any $\delta\in(0, 1)$, any action class $\mathcal{A}\subset \mathbb{R}^d$, and any expected rewards $\boldsymbol{\mu}\in\mathbb{R}^d$ are given. The reward distribution $\phi_s$ for each arm $s\in[d]$ is an $R$-sub-Gaussian tail distribution.
    With probability at least $1 - \delta$, an upper bound of the sample complexity $\tau$ of Algorithm \ref{OurAlgorithm} is
    \begin{align}
          8R^{2} \sum_{s = 1}^{d} \frac{V_{s}}{\Delta^{2}_{(s)}} \log \left( \frac{2K^{2}C^{2}}{\delta} \right) + Ad,
    \end{align}
    where $V_{s} = \max\limits_{\boldsymbol{\pi} \in \left\{ \boldsymbol{\pi} \in \mathcal{A} \ | \ \pi_{s} \neq \pi^{*}_{s} \right\}, \boldsymbol{\pi}' \in \mathcal{A}} \frac{\left| \pi_{s} - \pi'_{s} \right|}{\left| \pi^{*}_{s} - \pi_{s} \right|^{2}} \sum_{u = 1}^{d} \left| \pi_{u} - \pi'_{u} \right|$, $A = \max\limits_{s, u \in [d], \boldsymbol{\pi}, \boldsymbol{\pi} \in \mathcal{A}, \pi_{s} \neq \pi'_{s}} \frac{ \left| \pi_{u} - \pi'_{u} \right|}{\left| \pi_{s} - \pi'_{s} \right|}$, and $C$ is a constant determined by $R$, $\left\{V_{s}\right\}_{s = 1, \ldots, d}$,$\left\{ \Delta_{(s)} \right\}_{s = 1, \ldots, d}$, $A$, and $d$.
\end{theorem}
We show the proof in Appendix \ref{ProofSampleComplexity}. From \citet{nakamura2023thompson}, a lower bound of the sample complexity of the R-CPE-MAB can be written as 
\begin{align}
    \mathcal{O}\left( R^{2} \sum_{s = 1}^{d} \frac{1}{\Delta^{2}_{(s)}} \log \left( \frac{1}{\delta} \right) \right). 
\end{align}
Therefore, the CombGapE algorithm is optimal up to a problem-dependent constant factor. Especially for the pure exploration of the ordinary $d$-armed multi-armed bandit \citep{Audibert2010,Gabillon2012}, since $V_{s} = 2$ and $\Delta_{(s)} = \min\limits_{\boldsymbol{\pi} \in \mathcal{A} \setminus \left\{ \boldsymbol{\pi}^{*} \right\}} \left( \boldsymbol{\pi}^{*} - \boldsymbol{\pi} \right)^{\top}\boldsymbol{\mu}$, the upper bound is $16R^{2}\sum\limits_{s = 1}^{d} \frac{1}{\Delta^{2}_{(s)}} \log \left( \frac{1}{\delta} \right)$. This matches a lower bound shown in \citet{Kaufmann2016}, which is $\mathcal{O}\left( R^{2} \sum\limits_{s = 1}^{d} \frac{1}{\Delta^{2}_{(s)}} \log \left( \frac{1}{\delta} \right) \right)$.

\subsection{Comparison with Existing Works}
Some algorithms in the literature of transductive bandits can be used in the R-CPE-MAB. \citet{Fiez2019} proposed the RAGE (Randomized Adaptive Gap Elimination) algorithm whose upper bound of the sample complexity is 
\begin{align}
    \mathcal{O} \left( R^{2} \min_{\boldsymbol{\lambda} \in \Pi_{d}} \max_{i \in [K]\setminus \left\{ a^{*} \right\}  } \left\{ \frac{ \sum_{s = 1}^{d} \frac{\left| \pi^{*}_{s} - \pi_{s}^{i} \right|^{2}}{\lambda_{s}} }{\Delta_{i}^{2}} \right\} \log \left( \frac{1}{\Delta_{\mathrm{min}}} \right) \log\left(  \frac{1}{\delta} \right)  \right),
\end{align}
where $\Delta_{\mathrm{min}} = \min_{i \in [K]\setminus \{a^{*}\}} \left( \boldsymbol{\pi}^{*} - \boldsymbol{\pi} \right)^{\top} \boldsymbol{\mu}$. The RAGE algorithm has an extra logarithmic term compared to the lower bound shown in \cite{Fiez2019}, which is 
\begin{align}
    \mathcal{O} \left( R^{2} \min_{\boldsymbol{\lambda} \in \Pi_{d}} \max_{i \in [K]\setminus \left\{ a^{*} \right\}  } \left\{ \frac{ \sum_{s = 1}^{d} \frac{\left| \pi^{*}_{s} - \pi_{s}^{i} \right|^{2}}{\lambda_{s}} }{\Delta_{i}^{2}} \right\} \log\left(  \frac{1}{\delta} \right)  \right).
\end{align}
\citet{KatzSamuels2020} proposed an algorithm named the Peace algorithm whose order of the sample complexity upper bound is the same as RAGE. The Peace algorithm works only when the rewards of each arm follow a Gaussian distribution. This is a relatively stronger assumption than our assumption that each arm's reward follows an $R$-sub-Gaussian distribution. \par
\citet{nakamura2023thompson} proposed an algorithm named the GenTS-Explore (Generalized Thompson Sampling Explore) algorithm for the R-CPE-MAB problem. Its order of the sample complexity upper bound is the same as that of the CombGapE algorithm. One advantage the GenTS-Explore algorithm has over the CombGapE algorithm is that it works even when the size of the action class is exponentially large in $d$. However, it requires sampling many times from a Gaussian distribution each time step to estimate the true reward and is computationally heavy especially when the number of arms is large. It is an interesting future work to investigate whether there is an optimal algorithm even when the size of the action class is exponentially large in~$d$.  We see this in Section \ref{ExperimenRealWorldDataSection}. \par
In Section \ref{ExperimentSection}, we numerically show that the CombGapE algorithm significantly outperforms the RAGE, Peace, and GenTS-Explore algorithms.

\section{Experiment}\label{ExperimentSection}
Here, we show that the CombGapE algorithm outperforms existing algorithms in both synthetic and real-world data.
\subsection{Experiment on Synthetic Data}
Here, we numerically illustrate the behavior of the proposed CombGapE algorithm and existing methods with synthetic data. We compare the CombGapE algorithm with existing methods in the best action identification problem for the knapsack problem \citep{dantzig2007}. \par
In the knapsack problem, we have $d$ items. Each item $s\in[d]$ has a weight $w_s$ and value $v_s$. Also, there is a knapsack whose capacity is $W$ in which we put items. Our goal is to maximize the total value of the knapsack, not letting the total weight of the items exceed the capacity of the knapsack. Formally, the optimization problem is given as follows:
\begin{equation*}
\begin{array}{ll@{}ll}
\text{maximize}_{\boldsymbol{\boldsymbol{\pi}}\in\mathcal{A}}  & \sum_{s = 1}^{d}v_{s}\pi_s  &\\
\text{subject to}& \sum_{s = 1}^{d}\pi_s w_s \leq W, &
\end{array}
\end{equation*}
where $\pi_s$ denotes the number of item $s$ in the knapsack. Here, the weight of each item is known, but the value is unknown and therefore has to be estimated. In each time step, the player chooses an item $s$ and gets an observation of value $v_s$, which can be regarded as a random variable from an unknown distribution with mean $v_s$. \par
We generate the weight of each item uniformly from $\left\{5, 6, \ldots, 50 \right\}$ for our experiment. For each item $s \in [d]$, we generate $v_s$ as $v_s = w_s + x$, where $x$ is a uniform sample from $[-5, 5]$. We set the capacity of the knapsack as $W = 100$. As a prior knowledge of the problem, we assume that we know each $v_s$ is in $[w_s - 5, w_s +5]$, and use this prior knowledge to generate the action class $\mathcal{A}$ in the following procedure. We first generate a vector $\boldsymbol{v}'$ whose $s$-th element $v'_s$ is uniformly sampled from $[w_s -5, w_s + 5]$, and then solve the knapsack problem with $v'_s$ and add the obtained solution $\boldsymbol{\pi}$ to $\mathcal{A}$. We repeat this until 100 times, and therefore, $|\mathcal{A}|\leq 100$. Each time we choose an item $s$, we observe a value $v_s + x$ where $x$ is a noise from $\mathcal{N}(0, 1)$. We set $R=1$. We run experiments 30 times. \par
We show the result in Table \ref{knapsack_experiment}. We can see that the CombGapE algorithm using the arm selection strategy shown in (\ref{p(t)_def_2}) significantly outperforms existing methods. 

\begin{table}[t]
    \caption{The mean and standard deviation of sample complexity of each algorithm normalized by the sample complexity of the CompGapE algorithm for the knapsack problem. CombGapE (Naive) pulls the arm $a(t) = \argmin_{s \in \left\{ u\in [d] \ | \ \hat{\pi}_{u}(t) \neq \tilde{\pi}_{u}(t) \right\}} T_{s}(t)$ each round (see the discussion in Section \ref{limitation_of_Existing_Work}). }
    \centering
  \begin{tabular}{cccc} 
      & $d = 5$ & $d = 7$ & $d = 9$ \\
    \midrule
    CombGapE (Naive) &  $2.6 \pm 4.1$ & $3.0 \pm 3.3$ & $2.2 \pm 2.5$ \\
    GenTS-Explore & $65 \pm 81$ & $58 \pm 45$ & $74 \pm 75$ \\
    RAGE & \begin{tabular}{c}
          $1.8 \times 10^4 $ \\
          $\pm 3.5 \times 10^4$
    \end{tabular}
      & \begin{tabular}{c}
          $8.4 \times 10^3 $ \\
          $\pm 1.3 \times 10^4$
    \end{tabular} & \begin{tabular}{c}
          $2.0 \times 10^4 $ \\
          $\pm 4.0 \times 10^4$
    \end{tabular} \\
    Peace & $82 \pm 96$ & $39 \pm 29$& $55 \pm 49$ \\
    \bottomrule
  \end{tabular} 
  \label{knapsack_experiment}
\end{table}

\subsection{Experiment on Real-World Data} \label{ExperimenRealWorldDataSection}
Next, we numerically illustrate the behavior of the proposed CombGapE algorithm and
existing methods with real-world data. We compare the CombGapE algorithm with existing methods in the best action identification problem for the optimal transport problem \citep{villani2008}. \par
Optimal transport can be regarded as the cheapest plan to deliver resources from $m$ suppliers to $n$ consumers, where each supplier $i$ and consumer $j$ have supply $s_i$ and demand $d_j$, respectively. Let $\boldsymbol{\gamma}\in\mathbb{R}^{m \times n}_{\geq 0}$ be the cost matrix, where $\gamma_{ij}$ denotes the cost between supplier $i$ and demander $j$. Our objective is to find the optimal transportation matrix
\begin{equation} \label{OT_formulation}
    \boldsymbol{\pi}^{*} = \argmin_{\boldsymbol{\pi}\in \mathcal{G}(\boldsymbol{s}, \boldsymbol{d})} \sum_{i, j}\pi_{ij}\gamma_{ij},
\end{equation}
where 
\begin{align}\label{coupling_constraint}
    \mathcal{G}(\boldsymbol{s}, \boldsymbol{d}) \triangleq \left\{ \boldsymbol{\Pi} \in \mathbb{R}^{m \times n}_{\geq 0} \,\middle|\, \boldsymbol{\Pi} \boldsymbol{1}_n = \boldsymbol{s}, \boldsymbol{\Pi}^{\top} \boldsymbol{1}_m = \boldsymbol{d} \right\}.
\end{align}
Here, $\boldsymbol{s} = (s_1, \ldots, s_m)$ and $\boldsymbol{d} = (d_1, \ldots, d_n)$. 
$\pi_{ij}$ represents how much resources one sends from supplier $i$ to demander $j$. 
If we assume that the cost is unknown and changes stochastically, e.g., due to some traffic congestions, we can apply the R-CPE-MAB framework to the optimal transport problem, where we estimate the cost of each edge $\left(i, j\right)$ between supplier $i$ and consumer $j$. \par
In our experiment, we use cities in the eastern states of the United States. We have 9 suppliers and 9 demanders. For the 9 suppliers, we choose New York, Boston, Philadelphia, Washington, Harrisburg, Pittsburgh, Albany, Richmond, and Rochester. For the 9 demanders, we choose Charlotte, Columbia, Manchester, Atlanta, Norfolk, Cleveland, Portland, Columbus, and Watertown. For every $i \in \{1, \ldots, 9\}$ and $j \in \{1, \ldots, 9\}$, we obtain the time required to move from city $i$ to city $j$ from the Google Maps \footnote{https://www.google.co.jp/maps/} and set as $\gamma_{ij}$. The unit of measure is hours. We generate each element of $\boldsymbol{s}$ and $\boldsymbol{d}$ from a uniform distribution over $[0, 1]$, and normalize them so that $\sum_{i = 1}^{9} s_{i} = 1$ and $\sum_{j = 1}^{9} d_{j} = 1$ is satisfied. Note that since we have 9 suppliers and 9 demanders, we have 81 edges, and therefore, $d = 81$. \par
As a prior knowledge, for every $i \in \{1, \ldots, 9\}$ and $j \in \{ 1, \ldots, 9 \}$, we assume that we know the cost of edge $(i, j)$ is in the interval $[\gamma_{ij} - 1, \gamma_{ij} + 1]$, and use this knowledge to generate the action class $\mathcal{A}$ in the following procedure. We first generate a matrix $\boldsymbol{\gamma}^{\mathrm{dummy}}$ whose $(i, j)$ element is uniformly sampled from $[\gamma_{ij} - 1, \gamma_{ij} + 1]$. Then, we solve the optimal transport problem with $\gamma^{\mathrm{dummy}}$ and add the obtained solution to $\mathcal{A}$. We repeat this 1000 times, and therefore, $|\mathcal{A}| \leq 1000$. Each time we choose an edge $(i, j)$, we observe a value $\gamma_{ij} + x$, where $x$ is a random variable from a standard Gaussian distribution. Here, note that $R = 1$. We run experiments 100 times.  \par
We show the result in Table \ref{OT_experiment}. We have not included the results of TSExplore since it was computationally infeasible since the sampling operation must be performed many times from a normal distribution to estimate the true cost matrix in each round. We can see that the CombGapE algorithm using the arm selection strategy shown in (\ref{p(t)_def_2}) significantly outperforms existing methods. The CombGapE algorithm successfully outputted the true optimal action 100 times out of 100 experiments. 

\begin{table}[t]
    \caption{The mean and standard deviation of sample complexity of each algorithm normalized by the sample complexity of the CompGapE algorithm for the optimal transport problem. CombGapE (Naive) pulls the arm $a(t) = \argmin_{s \in \left\{ u\in [d] \ | \ \hat{\pi}_{u}(t) \neq \tilde{\pi}_{u}(t) \right\}} T_{s}(t)$ each round (see the discussion in Section \ref{limitation_of_Existing_Work}).}
    \centering
  \begin{tabular}{cc} 
      & $d = 81$ \\
    \midrule
    CombGapE (Naive) & 2.08 $\pm $ 1.55 \\
    RAGE & 25.7 $ \pm $ 16.6 \\
    PEACE & 29.7 $ \pm $ 23.5 \\
    \bottomrule
  \end{tabular} 
  \label{OT_experiment}
\end{table}

\section{Conclusion}\label{Conclusion}
We studied the R-CPE-MAB in the stochastic multi-armed bandit for the case where the size of the action set is polynomial with respect to the number of arms, which can be seen as a special case of the so-called transductive linear bandits. We proposed an optimal algorithm named the CombGapE algorithm, whose sample complexity upper bound matches the lower bound up to a problem-dependent constant factor. We numerically showed that the CombGapE algorithm outperforms existing methods significantly in both synthetic and real-world data.

\section*{Acknowledgement}
We thank Dr. Kevin Jamieson for his very helpful advice and comments on existing studies.
SN was supported by JST SPRING, Grant Number JPMJSP2108. MS was supported by JST CREST Grant Number JPMJCR18A2.

\bibliographystyle{apalike}
\bibliography{main}

\begin{thebibliography}{}

\bibitem[Audibert et~al., 2010]{Audibert2010}
Audibert, J.-Y., Bubeck, S., and Munos, R. (2010).
\newblock Best arm identification in multi-armed bandits.
\newblock In {\em The 23rd Conference on Learning Theory}, pages 41--53.

\bibitem[Auer et~al., 2002a]{Auer2002}
Auer, P., Cesa-Bianchi, N., and Fischer, P. (2002a).
\newblock Finite-time analysis of the multiarmed bandit problem.
\newblock {\em Machine Learning}, 47:235--256.

\bibitem[Auer et~al., 2002b]{Auer2002B}
Auer, P., Cesa-Bianchi, N., Freund, Y., and Schapire, R.~E. (2002b).
\newblock The nonstochastic multiarmed bandit problem.
\newblock {\em SIAM Journal on Computing}, 32(1):48--77.

\bibitem[Bubeck and Cesa-Bianchi, 2012]{Bubeck2012}
Bubeck, S. and Cesa-Bianchi, N. (2012).
\newblock Regret analysis of stochastic and nonstochastic multi-armed bandit problems.

\bibitem[Chen et~al., 2016]{LChen2016}
Chen, L., Gupta, A., and Li, J. (2016).
\newblock Pure exploration of multi-armed bandit under matroid constraints.
\newblock In {\em Proceedings of the 29th Conference on Learning Theory, {COLT} 2016, New York, USA, June 23-26, 2016}, volume~49 of {\em {JMLR} Workshop and Conference Proceedings}, pages 647--669. JMLR.org.

\bibitem[Chen et~al., 2017]{LChen2017}
Chen, L., Gupta, A., Li, J., Qiao, M., and Wang, R. (2017).
\newblock Nearly optimal sampling algorithms for combinatorial pure exploration.
\newblock In {\em Proceedings of the 2017 Conference on Learning Theory}, volume~65 of {\em Proceedings of Machine Learning Research}, pages 482--534. PMLR.

\bibitem[Chen et~al., 2014]{SChen2014}
Chen, S., Lin, T., King, I., Lyu, M.~R., and Chen, W. (2014).
\newblock Combinatorial pure exploration of multi-armed bandits.
\newblock In {\em Proceedings of the 27th International Conference on Neural Information Processing Systems - Volume 1}, page 379–387, Cambridge, MA, USA. MIT Press.

\bibitem[Dantzig and Mazur, 2007]{dantzig2007}
Dantzig, T. and Mazur, J. (2007).
\newblock {\em Number: The Language of Science}.
\newblock A Plume book. Penguin Publishing Group.

\bibitem[Du et~al., 2021a]{YihanDu2021}
Du, Y., Kuroki, Y., and Chen, W. (2021a).
\newblock Combinatorial pure exploration with bottleneck reward function.
\newblock In {\em Advances in Neural Information Processing Systems}, volume~34, pages 23956--23967. Curran Associates, Inc.

\bibitem[Du et~al., 2021b]{Du2021}
Du, Y., Kuroki, Y., and Chen, W. (2021b).
\newblock Combinatorial pure exploration with full-bandit or partial linear feedback.
\newblock In {\em AAAI Conference on Artificial Intelligence}.

\bibitem[Fiez et~al., 2019]{Fiez2019}
Fiez, T., Jain, L.~P., Jamieson, K.~G., and Ratliff, L.~J. (2019).
\newblock Sequential experimental design for transductive linear bandits.
\newblock In {\em Neural Information Processing Systems}.

\bibitem[Fomin et~al., 2015]{Fomin2015}
Fomin, F.~V., Philip, G., and Villanger, Y. (2015).
\newblock Minimum fill-in of sparse graphs: Kernelization and approximation.
\newblock {\em Algorithmica}, 71(1):1–20.

\bibitem[Gabillon et~al., 2012]{Gabillon2012}
Gabillon, V., Ghavamzadeh, M., and Lazaric, A. (2012).
\newblock Best arm identification: A unified approach to fixed budget and fixed confidence.
\newblock In {\em Proceedings of the 25th International Conference on Neural Information Processing Systems - Volume 2}, NIPS'12, page 3212–3220, Red Hook, NY, USA. Curran Associates Inc.

\bibitem[Gabillon et~al., 2016]{Gabillon16}
Gabillon, V., Lazaric, A., Ghavamzadeh, M., Ortner, R., and Bartlett, P. (2016).
\newblock Improved learning complexity in combinatorial pure exploration bandits.
\newblock In {\em Proceedings of the 19th International Conference on Artificial Intelligence and Statistics}, volume~51 of {\em Proceedings of Machine Learning Research}, pages 1004--1012, Cadiz, Spain. PMLR.

\bibitem[Gibbons, 1985]{Gibbons1985}
Gibbons, A. (1985).
\newblock {\em Algorithmic Graph Theory}.
\newblock Cambridge University Press.

\bibitem[Gutin et~al., 2001]{Gutin2001}
Gutin, G., Punnen, A., Barvinok, A., Gimadi, E., and Serdyukov, A. (2001).
\newblock The traveling salesman problem and its variations.

\bibitem[Jourdan et~al., 2021]{Jourdan2021}
Jourdan, M., Mutn\'y, M., Kirschner, J., and Krause, A. (2021).
\newblock Efficient pure exploration for combinatorial bandits with semi-bandit feedback.
\newblock In {\em Proceedings of the 32nd International Conference on Algorithmic Learning Theory}, volume 132 of {\em Proceedings of Machine Learning Research}, pages 805--849. PMLR.

\bibitem[Kalyanakrishnan and Stone, 2010]{Kalyanakrishnan2010}
Kalyanakrishnan, S. and Stone, P. (2010).
\newblock Efficient selection of multiple bandit arms: Theory and practice.
\newblock In {\em Proceedings of the 27th International Conference on International Conference on Machine Learning}, ICML'10, page 511–518, Madison, WI, USA. Omnipress.

\bibitem[Katz-Samuels et~al., 2020]{KatzSamuels2020}
Katz-Samuels, J., Jain, L., karnin, z., and Jamieson, K.~G. (2020).
\newblock An empirical process approach to the union bound: Practical algorithms for combinatorial and linear bandits.
\newblock In {\em Advances in Neural Information Processing Systems}, volume~33, pages 10371--10382. Curran Associates, Inc.

\bibitem[Kaufmann et~al., 2016]{Kaufmann2016}
Kaufmann, E., Capp\'{e}, O., and Garivier, A. (2016).
\newblock On the complexity of best-arm identification in multi-armed bandit models.
\newblock {\em J. Mach. Learn. Res.}, 17(1):1–42.

\bibitem[Kuroki et~al., 2020]{Kuroki2020}
Kuroki, Y., Xu, L., Miyauchi, A., Honda, J., and Sugiyama, M. (2020).
\newblock {Polynomial-Time Algorithms for Multiple-Arm Identification with Full-Bandit Feedback}.
\newblock {\em Neural Computation}, 32(9):1733--1773.

\bibitem[Labille et~al., 2021]{Labille2021}
Labille, K., Huang, W., and Wu, X. (2021).
\newblock Transferable contextual bandits with prior observations.
\newblock In {\em Advances in Knowledge Discovery and Data Mining}, pages 398--410, Cham. Springer International Publishing.

\bibitem[Nakamura and Sugiyama, 2023]{nakamura2023thompson}
Nakamura, S. and Sugiyama, M. (2023).
\newblock Thompson sampling for real-valued combinatorial pure exploration of multi-armed bandit.
\newblock arXiv:2308.10238.

\bibitem[Pettie and Ramachandran, 2002]{Pettie2002}
Pettie, S. and Ramachandran, V. (2002).
\newblock An optimal minimum spanning tree algorithm.
\newblock {\em J. ACM}, 49(1):16–34.

\bibitem[Philip et~al., 2009]{Philip2009}
Philip, G., Raman, V., and Sikdar, S. (2009).
\newblock Solving dominating set in larger classes of graphs: Fpt algorithms and polynomial kernels.
\newblock In {\em Algorithms - ESA 2009}, pages 694--705, Berlin, Heidelberg. Springer Berlin Heidelberg.

\bibitem[Pochet and Wolsey, 2010]{Pochet2010}
Pochet, Y. and Wolsey, L.~A. (2010).
\newblock {\em Production Planning by Mixed Integer Programming}.
\newblock Springer Publishing Company, Incorporated, 1st edition.

\bibitem[Rejwan and Mansour, 2020]{Rejwan2020}
Rejwan, I. and Mansour, Y. (2020).
\newblock Top-\$k\$ combinatorial bandits with full-bandit feedback.
\newblock In {\em International Conference on Algorithmic Learning Theory}.

\bibitem[Rivasplata, 2012]{Rivasplata2012}
Rivasplata, O. (2012).
\newblock Subgaussian random variables : An expository note.

\bibitem[Sniedovich, 2006]{Sniedovich2006}
Sniedovich, M. (2006).
\newblock Dijkstra's algorithm revisited: the dynamic programming connexion.
\newblock {\em Control and Cybernetics}, 35:599--620.

\bibitem[Soare et~al., 2014]{Soare2014}
Soare, M., Lazaric, A., and Munos, R. (2014).
\newblock Best-arm identification in linear bandits.
\newblock {\em Advances in Neural Information Processing Systems}, 1.

\bibitem[Villani, 2008]{villani2008}
Villani, C. (2008).
\newblock {\em Optimal Transport: Old and New}.
\newblock Grundlehren der mathematischen Wissenschaften. Springer Berlin Heidelberg.

\bibitem[Wang and Zhu, 2022]{WangAndZhu}
Wang, S. and Zhu, J. (2022).
\newblock Thompson sampling for (combinatorial) pure exploration.
\newblock In {\em Proceedings of the 39 th International Conference on Machine Learning}, Baltimore, Maryland, USA.

\bibitem[Xu et~al., 2018]{Xu2018}
Xu, L., Honda, J., and Sugiyama, M. (2018).
\newblock A fully adaptive algorithm for pure exploration in linear bandits.
\newblock In {\em Proceedings of the Twenty-First International Conference on Artificial Intelligence and Statistics}, volume~84 of {\em Proceedings of Machine Learning Research}, pages 843--851. PMLR.

\bibitem[Yang and Gao, 2021]{Yang2021}
Yang, S. and Gao, Y. (2021).
\newblock An optimal algorithm for the stochastic bandits while knowing the near-optimal mean reward.
\newblock {\em IEEE Transactions on Neural Networks and Learning Systems}, 32(5):2285--2291.

\end{thebibliography}


\newpage
\appendix
\section{Situations where we can assume the size of $\mathcal{A}$ is polynomial in $d$}\label{PolynomialAppendix}
In our study, we assume that the size of the action space $\mathcal{A}$ is polynomial in $d$ so that we can run a search algorithm in the action space to choose which arm to pull in each time step. For instance, if we think of the shortest path problem, this assumption holds when the graph is sparse \citep{Philip2009,Fomin2015}, and only a reasonably small number of actions (path) have to be compared. Alternatively, even though it is nearly impossible to identify the best action $\boldsymbol{\pi}^{*}$ which maximizes $\boldsymbol{\mu}^{\top}\boldsymbol{\pi}$ from $\mathbb{R}^d$ due to the uncertainty of $\boldsymbol{\mu}$, it may be sufficient to identify the best action from a set of candidate $\mathcal{A}$. In order to construct $\mathcal{A}$, one may use some prior knowledge of each arm, which is sometimes obtainable in the real world \citep{Labille2021, Yang2021}. For instance, if we model the real world with the optimal transport problem, we may know in advance that some edges are clearly more costly than others. For example, the distance from New York to San Francisco is clearly farther than from New York to Boston. Also, nowadays, the approximate time required to travel between two points on the globe can be easily obtained from databases. Based on this prior knowledge, a realistic transportation method can be narrowed down from all the solutions to the optimal transport problem.  \par

\section{Limitation of existing works} \label{EWLimitationAppendix}
Here, we discuss the limitation of other existing works in CPE-MAB \citep{YihanDu2021,LChen2016,LChen2017}. \citet{YihanDu2021} introduced an algorithm named GenLUCB, which is an algorithm for CPE-MAB. In each round, it computes a lower and an upper bound for each arm's reward, and uses the maximization oracle to find the action $M_t$ with the maximum pessimistic reward from $\mathcal{A}$ and the action $\tilde{M}_t$ with the maximum optimistic reward from $\mathcal{A} \setminus \{M_t\}$. Then, we play the arm $p(t)$ with the maximum confidence radius from $M_t \cup \tilde{M}_t$, which is the arm with the least number of times it was pulled among $M_t \cup \tilde{M}_t$. The algorithm repeats this procedure until a certain stopping condition is satisfied. However, as we discussed in Section \ref{ArmSelectionStrategySection}, it is not clear what sets $M_t$, $\Tilde{M}_t$, and $M_t \cup \tilde{M}_t$ are in R-CPE-MAB, since the action class is no longer binary in R-CPE-MAB. \citet{LChen2016} introduces an algorithm for CPE-MAB when the action class $\mathcal{A}$ is a matroid, such as the top-$k$ arm identification. However, in R-CPE-MAB, $\mathcal{A}$ is not necessarily a matroid, and the applicability is limited for our model. \citet{LChen2017} introduced the NaiveGapElim algorithm, which is an algorithm for CPE-MAB. The sampling procedure is a little different from CPE-MAB explained in Section \ref{IntroductionSection}, but essentially the same. Each time step, it computes a vector $\boldsymbol{m} = (m_1, \ldots, m_d)^{\top}$ by solving a certain optimization problem, and pulls arm $s$ $m_s$ times. The algorithm is evaluated by the total number of arms the algorithm pulled when it outputs the best action. Therefore, it is essentially the same setting as our model. However, $\boldsymbol{m}$ is computed by considering the symmetric difference between actions, which can only be defined when actions are binary.
\par

\section{Proof of Proposition \ref{KeyProposition}} \label{ProofKeyProposition}
Here, we prove Proposition \ref{KeyProposition}. First, we recall the Hoeffding's inequality, which is a concentration inequality of sub-Gaussian random variables. The proof of it can be seen in Lemma 6 in \citet{SChen2014}.
\begin{lemma}[Hoeffding's inequality \citep{SChen2014}] \label{HoeffdingInequality}
    Let $X_1, \ldots, X_n$ be $n$ independent random variables such that, for each $i\in [n]$, random variable $X_i - \mathbb{E}[X_i]$ is $R$-sub-Gaussian distributed, i.e., $\forall t\in\mathbb{R}, \mathbb{E}[\exp (tX_i - t\mathbb{E}[X_i])] \leq \exp(R^2t^2/2)$. Let $\bar{X} = \sum_{i = 1}^{n}X_i$ denote the average of these random variables. Then, for any $\lambda > 0$, we have
    \begin{equation}
        \Pr\left[ | \bar{X} - \mathbb{E}[\bar{X}]| \geq \lambda \right] \leq 2\exp\left( -\frac{n\lambda^2}{2R^2} \right)
    \end{equation}
\end{lemma}
Next, we note that if $X_1$ and $X_2$ are $R_1$-sub-Gaussian and $R_2$-sub-Gaussian random variables, respectively, $\alpha_1 X_1 + \alpha_2 X_2$ is $(\alpha R_1 + \alpha R_2)$-sub-Gaussian. Here, $\alpha_1, \alpha_2 \in \mathbb{R}$. The proof can be seen in \citet{Rivasplata2012}. \par
Now, we prove Proposition \ref{KeyProposition}.
\begin{proof}[Proof of Proposition \ref{KeyProposition}]
    Fix, $t$ and $k, l \in [K]$. From Hoeffding's inequality (Lemma \ref{HoeffdingInequality}), we have 
    \begin{align}
        &\Pr \left( \left|\Delta_{kl} - \hat{\Delta}_{kl}(t) \right| \leq \beta_{kl}(t) \right) \nonumber \\
   \leq & 2 \exp \left( - \frac{2 {\beta_{kl}(t)}^2}{\sum_{s = 1}^{d} \frac{(\pi^k_s - \pi^l_s)^2 R^2}{T_s(t)}} \right) \nonumber \\
   \leq & \frac{\delta}{K^2 t^2}.
    \end{align}
    From union bounds, for any $t\in\mathbb{N}$ and $k,l \in [K]$, we have
    \begin{align}
        & \Pr \left( \left|\Delta_{kl} - \hat{\Delta}_{kl}(t) \right| \leq \beta_{kl}(t) \right) \nonumber \\
    \leq& \sum_{k, l \in[K]}\sum_{t = 1}^{\infty} \frac{\delta}{K^2 t^2} \nonumber \\
    \leq& \delta.
    \end{align}
\end{proof}

\section{Proof of Proposition \ref{ArmSelectionStrategyProposition}} \label{ArmSelectionStrategyPropositionProof}
Here, we prove Proposition \ref{ArmSelectionStrategyProposition}.
\begin{proof}\
    For any $q\neq p(t)$, we have
    \begin{align}
        \sum_{s = 1}^{d}\frac{(\pi^{k}_s - \pi^{l}_s)^2}{T_k(t) + \boldsymbol{1}[s = q]} \geq \sum_{s = 1}^{d}\frac{(\pi^{k}_s - \pi^{l}_s)^2}{T_k(t) + \boldsymbol{1}[s = p(t)]}.
    \end{align}
    Thus, 
     \begin{equation}
            \frac{(\pi^{k}_q - \pi^{l}_q)^2}{T_q(t) + 1} + \frac{(\pi^{k}_{p(t)} - \pi^{l}_{p(t)})^2}{T_{p(t)}(t)} 
            \geq \frac{(\pi^{k}_q - \pi^{l}_q)^2}{T_q(t)} + \frac{(\pi^{k}_{p(t)} - \pi^{l}_{p(t)})^2}{T_{p(t)}(t) + 1}, \nonumber
    \end{equation}
    which implies
    \begin{equation}
            \frac{(\pi^{k}_{p(t)} - \pi^{l}_{p(t)})^2}{T_{p(t)}(t)(T_{p(t)}(t) + 1)} \geq  \frac{(\pi^{k}_{q} - \pi^{l}_{q})^2}{T_{q}(t)(T_{q}(t) + 1)}, \nonumber
    \end{equation}
    for all $q \neq p(t)$
\end{proof}

\section{Proof of Proposition x\ref{identification_accuracy}}\label{ProofIdentificationAccuracy}
\begin{proof}
    Let $\mathcal{E}$ be an event defined as follows:
\begin{equation}
    \mathcal{E} = \{ \forall t > 0, \ \forall i, j\in [K], |\Delta_{ij} - \hat{\Delta}_{ij}(t)| \leq \beta_{ij}(t) \}. \nonumber
\end{equation}
    From Proposition \ref{KeyProposition}, this event holds with probability at least 1 - $\delta$.
    Let $\tau$ be the stopping round of CombGapE. If $\Delta(a^{*}, a_{\mathrm{out}}) > 0$ holds, then we have
\begin{equation}
    \Delta_{a^{*} a_{\mathrm{out}}} > 0 \geq B(\tau) \geq \hat{\Delta}_{a^* a_{\mathrm{out}}}(\tau) + \beta_{a^* a_{\mathrm{out}}}(\tau). \nonumber
\end{equation}
The second inequality holds for stopping condition $B(\tau) \leq 0$ and the last follows from the definition of $B(\tau)$. From this inequality, we can see that $\Delta_{a^* a_{\mathrm{out}}}> 0$ means that event $\mathcal{E}$ does not occur. Thus, the probability that CombGapE returns such actions is 
\begin{equation}
    \Pr\left[ \boldsymbol{\pi}_{\mathrm{out}} \neq \boldsymbol{\pi}^{*} \right] \geq \Pr\left[ \mathcal{E} \right] \geq 1 - \delta.
\end{equation}
\end{proof}

\section{Proof of Theorem \ref{sample_complexity}} \label{ProofSampleComplexity}
Define $A = \max\limits_{s, u \in [d], \boldsymbol{\pi}, \boldsymbol{\pi} \in \mathcal{A}, \pi_{s} \neq \pi'_{s}} \frac{ \left| \pi_{u} - \pi'_{u} \right|}{\left| \pi_{s} - \pi'_{s} \right|}$ and $V_{p} = \max\limits_{\boldsymbol{\pi} \in \left\{ \boldsymbol{\pi} \in \mathcal{A} \ | \ \pi_{p} \neq \pi^{*}_{p} \right\}, \boldsymbol{\pi}' \in \mathcal{A}} \frac{\left| \pi_{p} - \pi'_{p} \right|}{\left| \pi^{*}_{p} - \pi_{p} \right|^{2}} \sum_{u = 1}^{d} \left| \pi_{u} - \pi'_{u} \right|$. Then, we have the following lemmas.
\begin{lemma} \label{LemmaFirst}
    If arm $p$ is pulled with $T_{p}(t) \geq R^{2}  \frac{8 V_{p}}{\Delta^{2}_{(p)}} \log \left( \frac{2K^{2}t^{2}}{\delta} \right) + A$, then $\boldsymbol{\pi}^{i(t)} = \boldsymbol{\pi}^{*}$.
\end{lemma}
\begin{proof}
    Firstly, since arm $p$ is pulled, we have 
    \begin{align}
        \frac{\left| \pi^{i(t)}_{p} - \pi_{p}^{j(t)} \right|^{2}}{T_{p}(t)\left( T_{p}(t) + 1 \right)} \geq \frac{\left| \pi^{i(t)}_{u} - \pi_{u}^{j(t)} \right|^{2}}{T_{u}(t)\left( T_{u}(t) + 1 \right)} \nonumber
    \end{align}
    for any $u \in [d]$. Therefore, since $\frac{1}{T_{u}(t)} \leq \frac{1}{ \frac{\left| \pi^{i(t)}_{u} - \pi^{j(t)}_{u} \right|}{\left| \pi_{p}^{i(t)} - \pi_{p}^{j(t)}  \right|} T_{p}(t) -  1} $, we have 
    \begin{align}
        \sum_{u = 1}^{d} \frac{\left| \pi^{i(t)}_{u} - \pi^{j(t)}_{u} \right|^{2}}{T_{u}(t)} \leq \left| \pi^{i(t)}_{p} - \pi^{j(t)}_{p} \right| \sum_{u = 1}^{d} \frac{ \left| \pi^{i(t)}_{u} - \pi^{j(t)}_{u} \right|}{T_{p}(t) - \frac{\left| \pi^{i(t)}_{p} - \pi^{j(t)}_{p} \right|}{\left|\pi^{i(t)}_{u} - \pi^{j(t)}_{u}\right|}} \label{SumUpperBound}
    \end{align}
    We prove the lemma by contradiction. Assume that $\boldsymbol{\pi}^{i(t)} \neq \boldsymbol{\pi}^{*}$. We have either $\pi_{p}^{i(t)} \neq \pi_{p}^{*}$ or $\pi_{p}^{j(t)} \neq \pi_{p}^{*}$. 
    If $\pi^{i(t)}_{p} \neq \pi^{*}_{p}$,
    \begin{align}
        \frac{\left( \boldsymbol{\pi}^{*} - \boldsymbol{\pi}^{i(t)} \right)^{\top} \boldsymbol{\mu}}{\left| \pi^{*}_{p} - \pi^{i(t)}_{p} \right|} 
         \leq & \frac{\left( \boldsymbol{\pi}^{*} - \boldsymbol{\pi}^{i(t)} \right)^{\top}\hat{\mu}(t)}{\left| \pi^{*}_{p} - \pi^{i(t)}_{p} \right|} + \frac{\beta_{a^{*}, i(t)}}{\left| \pi_{p}^{*} - \pi_{p}^{i(t)} \right|} \nonumber \\
         \leq & \frac{\left( \boldsymbol{\pi}^{j(t)} - \boldsymbol{\pi}^{i(t)} \right)^{\top}\hat{\mu}(t)}{\left| \pi^{*}_{p} - \pi^{i(t)}_{p} \right|} + \frac{\beta_{j(t), i(t)}}{{\left| \pi_{p}^{*} - \pi_{p}^{i(t)} \right|}} \nonumber \\
         \leq & \frac{\beta_{j(t), i(t)}}{{\left| \pi_{p}^{*} - \pi_{p}^{i(t)} \right|}} \nonumber \\
         \leq & R \sqrt{\frac{1}{2} \frac{1}{\left| \pi_{p}^{*} - \pi_{p}^{i(t)} \right|^{2}} \sum_{u = 1}^{d} \frac{\left| \pi^{i(t)}_{u} - \pi^{j(t)}_{u} \right|^{2}}{T_{u}(t)} \log \left( \frac{2K^{2}t^{2}}{\delta} \right)} \nonumber \\
         \leq & R \sqrt{\frac{1}{2} \frac{\left| \pi^{i(t)}_{p} - \pi^{j(t)}_{p} \right|}{\left| \pi^{*}_{p} - \pi^{i(t)}_{p}  \right|^{2}}  \sum_{u = 1}^{d} \frac{\left| \pi^{i(t)}_{u} - \pi^{j(t)}_{u} \right|}{T_{p}(t) - \frac{\left| \pi^{i(t)}_{p} - \pi^{j(t)}_{p} \right|}{\left|\pi^{i(t)}_{u} - \pi^{j(t)}_{u}\right|}} \log \left( \frac{2K^{2}t^{2}}{\delta} \right)} \nonumber \\
         \leq & \frac{\Delta_{(p)}}{4}. \nonumber \\
            < & \Delta_{(p)} \label{Eq_i_star_contradiction} 
    \end{align}
    However, since $\boldsymbol{\pi}^{*} \neq \boldsymbol{\pi}^{i(t)}$, $\Delta_{(p)} \leq \frac{\left(\boldsymbol{\pi}^{*} - \boldsymbol{\pi}^{i(t)} \right)^{\top} \boldsymbol{\mu}}{\left| \pi^{*}_{p} - \pi^{i(t)}_{p} \right|}$, and this contradicts (\ref{Eq_i_star_contradiction}). Therefore, $\boldsymbol{\pi}^{*} = \boldsymbol{\pi}^{i(t)}$. \par 
    If $\pi_{p}^{j(t)} \neq \pi_{p}^{*}$,
    \begin{align}
        &\frac{\left( \boldsymbol{\pi}^{*} - \boldsymbol{\pi}^{j(t)} \right)^{\top}\boldsymbol{\mu}}{\left| \pi_{p}^{*} - \pi_{p}^{j(t)} \right|} \nonumber \\
         = & \frac{\left( \boldsymbol{\pi}^{*} - \boldsymbol{\pi}^{i(t)} \right)^{\top} \boldsymbol{\mu}}{\left| \pi_{p}^{*} - \pi_{p}^{j(t)} \right|} + \frac{\left( \boldsymbol{\pi}^{i(t)} - \boldsymbol{\pi}^{j(t)} \right)^{\top} \boldsymbol{\mu}}{\left| \pi_{p}^{*} - \pi_{p}^{j(t)} \right|} \nonumber \\
         \leq & \frac{\left( \boldsymbol{\pi}^{*} - \boldsymbol{\pi}^{i(t)} \right)^{\top} \hat{\boldsymbol{\mu}}}{\left| \pi_{p}^{*} - \pi_{p}^{j(t)} \right|} + \frac{\beta_{a^{*}, i(t)}}{\left| \pi_{p}^{*} - \pi_{p}^{j(t)} \right|} + \frac{\left( \boldsymbol{\pi}^{i(t)} - \boldsymbol{\pi}^{j(t)} \right)^{\top} \hat{\boldsymbol{\mu}}}{\left| \pi_{p}^{*} - \pi_{p}^{j(t)} \right|} + \frac{\beta_{i(t), j(t)}}{\left| \pi_{p}^{*} - \pi_{p}^{j(t)} \right|} \nonumber \\
         \leq & \frac{\left( \boldsymbol{\pi}^{j(t)} - \boldsymbol{\pi}^{i(t)} \right)^{\top} \hat{\boldsymbol{\mu}}}{\left| \pi_{p}^{*} - \pi_{p}^{j(t)} \right|} + \frac{\beta_{j(t), i(t)}}{\left| \pi_{p}^{*} - \pi_{p}^{j(t)} \right|} + \frac{\left( \boldsymbol{\pi}^{i(t)} - \boldsymbol{\pi}^{j(t)} \right)^{\top} \hat{\boldsymbol{\mu}}}{\left| \pi_{p}^{*} - \pi_{p}^{j(t)} \right|} + \frac{\beta_{i(t), j(t)}}{\left| \pi_{p}^{*} - \pi_{p}^{j(t)} \right|} \nonumber \\
         = & 2\frac{\beta_{i(t), j(t)}}{\left| \pi_{p}^{*} - \pi_{p}^{j(t)} \right|} \nonumber \\
         = & 2 R \sqrt{\frac{1}{2} \frac{1}{\left| \pi_{p}^{*} - \pi_{p}^{j(t)} \right|^{2}} \sum_{u = 1}^{d} \frac{\left| \pi^{i(t)}_{u} - \pi^{j(t)}_{u} \right|^{2}}{T_{u}(t)} \log \left( \frac{2K^{2}t^{2}}{\delta} \right)} \nonumber \\
         \leq & 2R \sqrt{\frac{1}{2} \frac{\left| \pi^{j(t)}_{p} - \pi^{i(t)}_{p} \right|}{\left| \pi^{*}_{p} - \pi^{j(t)}_{p}  \right|^{2}}  \sum_{u = 1}^{d} \frac{\left| \pi^{j(t)}_{u} - \pi^{i(t)}_{u} \right|}{T_{p}(t) - \frac{\left| \pi^{j(t)}_{p} - \pi^{i(t)}_{p} \right|}{\left|\pi^{j(t)}_{u} - \pi^{i(t)}_{u}\right|}} \log \left( \frac{2K^{2}t^{2}}{\delta} \right)} \nonumber \\
         \leq & \frac{\Delta_{(p)}}{2} \nonumber \\
         < & \Delta_{(p)}. \label{Eq_j_star_contradiction}
    \end{align}
    However, since $\boldsymbol{\pi}^{*} \neq \boldsymbol{\pi}^{j(t)}$, $\Delta_{(p)} \leq \frac{\left(\boldsymbol{\pi}^{*} - \boldsymbol{\pi}^{j(t)} \right)^{\top} \boldsymbol{\mu}}{\left| \pi^{*}_{p} - \pi^{j(t)}_{p} \right|}$, and this contradicts (\ref{Eq_j_star_contradiction}). Therefore, $\boldsymbol{\pi}^{*} = \boldsymbol{\pi}^{i(t)}$. \par
\end{proof}

\begin{lemma} \label{LemmaSecond}
    At round $t$, an arm $p$ will not be pulled with $ T_{p}(t) \geq R^{2}  \frac{8 V_{p}}{\Delta^{2}_{(p)}} \log \left( \frac{2K^{2}t^{2}}{\delta} \right) +~A$.
\end{lemma}
\begin{proof}
    We prove this by contradiction. Assume that arm $p$ is pulled with $T_{p}(t) \geq \frac{8 V_{p}}{\Delta^{2}_{(p)}} \log \left( \frac{2K^{2}t^{2}}{\delta} \right) +~A $.
    We know that $\boldsymbol{\pi}^{i(t)} = \boldsymbol{\pi}^{*}$, and $\pi^{*}_{p} \neq \pi^{j(t)}_{p}$ is satisfied. \par
    If $\pi^{*}_{p} \neq \pi^{j(t)}_{p}$, then
    \begin{align}
        & - \frac{\left( \boldsymbol{\pi}^{i(t)} - \boldsymbol{\pi}^{j(t)} \right)^{\top} \hat{\boldsymbol{\mu}}(t)}{\left| \pi^{*}_{p} - \pi^{j(t)}_{p} \right|} + \frac{\beta_{i(t), j(t)}(t)}{\left| \pi^{*}_{p} - \pi^{j(t)}_{p} \right|} \nonumber \\
         \leq & -\frac{\left( \boldsymbol{\pi}^{i(t)} - \boldsymbol{\pi}^{j(t)} \right)^{\top} \boldsymbol{\mu}}{\left| \pi^{*}_{p} - \pi^{j(t)}_{p} \right|} + \frac{2\beta_{i(t), j)(t)}}{\left| \pi^{*}_{p} - \pi^{j(t)}_{p} \right|} \nonumber \\
         = & -\frac{\left( \boldsymbol{\pi}^{*} - \boldsymbol{\pi}^{j(t)} \right)^{\top} \boldsymbol{\mu}}{\left| \pi^{*}_{p} - \pi^{j(t)}_{p} \right|} + \frac{2\beta_{i(t), j)(t)}}{\left| \pi^{*}_{p} - \pi^{j(t)}_{p} \right|} \nonumber \\
         \leq & - \Delta_{(p)} + R\sqrt{\frac{1}{2} \frac{1}{\left| \pi^{i(t)}_{p} - \pi^{j(t)}_{p} \right|^{2}} \sum_{u = 1}^{d} \frac{\left(\pi^{i(t)}_{u} - \pi^{j(t)}_{u} \right)^{2}}{T_{u}(t)} \log \frac{2K^{2}t^{2}}{\delta}} \nonumber \\
         \leq & -\Delta_{(p)} + \frac{\Delta_{(p)}}{ 4 } \nonumber \\
         < & 0,
    \end{align}
    where the first equality holds because of Lemma \ref{LemmaFirst}.
    However, this contradicts the fact that the algorithm does not terminate. Therefore, arm $p$ will not be pulled. \par
\end{proof}
From, Lemma \ref{LemmaSecond}, we can see that 
\begin{align}
    t = \sum_{s = 1}^{d} T_{s}(t) \leq 8R^{2} \sum_{s = 1}^{d} \frac{V_{s}}{\Delta^{2}_{(s)}} \log \left( \frac{2K^{2}t^{2}}{\delta} \right) + Ad.
\end{align}
Therefore, we can confirm that 
\begin{align}
     t \leq \sum_{s = 1}^{d} T_{s}(t) \leq 8R^{2} \sum_{s = 1}^{d} \frac{V_{s}}{\Delta^{2}_{(s)}} \log \left( \frac{2K^{2}C^{2}}{\delta} \right) + Ad,
\end{align}
where $C$ is a constant determined by $R$, $\left\{V_{s} \right\}_{s = 1, \ldots, d}$ ,$\left\{ \Delta_{(s)} \right\}_{s = 1, \ldots, d}$, $A$, $K$, $\delta$, and $d$.
\end{document}